\documentclass{article}

\usepackage{arxiv}

\usepackage[utf8]{inputenc} \usepackage[T1]{fontenc}    \usepackage[pdfencoding=auto, psdextra]{hyperref}       \usepackage{url}            \usepackage{booktabs}       \usepackage{amsmath,amsthm,amssymb,amsfonts}       \usepackage{nicefrac}       \usepackage{microtype}      \usepackage{lipsum}
\usepackage{graphicx}
\graphicspath{ {./images/} }
\usepackage{paracol}

\usepackage{graphicx}
\usepackage{textcomp}
\usepackage{caption}
\usepackage{algorithm}
\usepackage[noend]{algorithmic}
\usepackage{cite}
\usepackage{float}
\usepackage{caption,subcaption}
\usepackage{lineno}

\usepackage{siunitx}

\newcommand{\norm}[1]{\left\lVert#1\right\rVert}
\newcommand{\td}{\mathrm{td}}\newcommand{\bfx}{{\bf x}}\newcommand{\bfb}{{\bf b}}\DeclareMathOperator*{\argmax}{argmax}
\newtheorem{theorem}{Theorem}
\newtheorem{lemma}{Lemma}
\usepackage{mathtools}

\newcommand{\ttitle}{Matrix tri-factorization over the tropical semiring}

\newcommand{\tkeywords}{tropical semiring, tri-factorization, network structure analysis, four-partition network}
 
\title{\ttitle}

\author{
Amra Omanović \\
  Faculty of Computer and Information Science\\
  University of Ljubljana\\
  Večna pot 113, 1000 Ljubljana, Slovenia\\
  \texttt{amra.omanovic@fri.uni-lj.si} \\
\And
Polona Oblak \\
  Faculty of Computer and Information Science\\
  University of Ljubljana\\
  Večna pot 113, 1000 Ljubljana, Slovenia\\
  \texttt{polona.oblak@fri.uni-lj.si} \\
\And
Tomaž Curk \\
Faculty of Computer and Information Science\\
  University of Ljubljana\\
  Večna pot 113, 1000 Ljubljana, Slovenia\\
  \texttt{tomaz.curk@fri.uni-lj.si} \\
}

\date{}

\begin{document}
\maketitle

\begin{abstract}
Tropical semiring has proven successful in several research areas, including optimal control, bioinformatics, discrete event systems, or solving a decision problem. In previous studies, a matrix two-factorization algorithm based on the tropical semiring has been applied to investigate bipartite and tripartite networks. Tri-factorization algorithms based on standard linear algebra are used for solving tasks such as data fusion, co-clustering, matrix completion, community detection, and more.
However, there is currently no tropical matrix tri-factorization approach, which would allow for the analysis of multipartite networks with a high number of parts.
To address this, we propose the \texttt{triFastSTMF} algorithm, which performs tri-factorization over the tropical semiring. We apply it to analyze a four-partition network structure and recover the edge lengths of the network. We show that \texttt{triFastSTMF} performs similarly to \texttt{Fast-NMTF} in terms of approximation and prediction performance when fitted on the whole network. When trained on a specific subnetwork and used to predict the whole network, \texttt{triFastSTMF} outperforms \texttt{Fast-NMTF} by several orders of magnitude smaller error. The robustness of \texttt{triFastSTMF} is due to tropical operations, which are less prone to predict large values compared to standard operations.  \end{abstract}

\keywords{\tkeywords}

\section{Introduction}
Matrix factorization methods embed data into a latent space using a two-factorization or tri-factorization approach, depending on the number of low-dimensional factor matrices required for the specific task. Matrix factorization methods can help solve problems in recommender systems~\cite{koren2009matrix}, pattern recognition~\cite{liu2006nonnegative}, data fusion~\cite{datafusion}, network structure analysis~\cite{linear_regression}, and similar. In many of these scenarios, two-factorization achieves state-of-the-art results. However, there are cases where tri-factorization outperforms two-factorization, such as in intermediate data fusion~\cite{datafusion}, where tri-factorization is used to fuse multiple data sources to improve the predictive power of the model.

Matrix factorization methods employ different types of operations to compute the factor matrices~\cite{latitude, capricorn, subalgs}. Most matrix factorization methods are based on standard linear algebra, such as non-negative matrix factorization~\cite{nmf} (\texttt{NMF}), binary matrix factorization~\cite{bmf} (\texttt{BMF}), probabilistic NMF~\cite{mnih2008probabilistic} (\texttt{PMF}), while some novel approaches such as \texttt{STMF}~\cite{stmf} and \texttt{FastSTMF}~\cite{faststmf} are based on the tropical semiring.

The \emph{$(\max,+)$ semiring} or \emph{tropical semiring} $\mathbb{R}_{\max}$ is the set $\mathbb{R} \cup \{-\infty\}$, equipped with $\max$ as addition ($\oplus$), and $+$ as multiplication ($\otimes$). 
For example, $2\oplus3=3$ and $1\otimes1=2$. 
Throughout the paper, the symbols ``$+$'' and ``$-$'' refer to standard operations of addition and subtraction.
The renowned \texttt{NMF} method~\cite{nmf} is based on the element-wise sum, which results in the ``parts-of-whole'' interpretation of factor matrices. On the contrary, tropical or $(\max,+)$ factorization uses the maximum operator, which results in a ``winner-takes-it-all'' interpretation~\cite{cancer}.
Matrix factorization approaches using tropical semiring demonstrated their robustness against overfitting and achieved predictive performance comparable to techniques that use standard linear algebra. Moreover, they also reveal different patterns, as we have demonstrated in our previous studies~\cite{stmf,faststmf}.

Tropical semirings have various applications in network structure analysis and other research areas~\cite{nonlinear_recomm, semiring_rank_mf, tropical_geometry}. 
Multiplication and addition of a similar $(\min, +)$ semiring enable mapping local edge information to global information on the \textit{shortest paths}, while the $(\max, +)$ semiring describes the \textit{longest path} problem. In our work, we are interested in an inverse problem that infers information about edges from potentially noisy or incomplete information~\cite{linear_regression}. 
To the best of our knowledge, there is no matrix tri-factorization method based on the tropical semiring. Thus, we propose the first tropical tri-factorization method, called \texttt{triFastSTMF}, which introduces a third factor matrix.
The proposed \texttt{triFastSTMF} can be used for various tasks that involve a \textit{single} data source. Our GitHub repository \url{https://github.com/Ejmric/triFastSTMF} provides the source code and data required to replicate our experiments. We demonstrate the applicability of \texttt{triFastSTMF} in edge approximation and prediction in a four-partition network. Moreover, this work sets the foundation for future research aimed at creating a tropical data fusion model capable of combining \textit{multiple} data sources.

The paper is divided into the following sections. Section~\ref{sec:relatedwork} describes the related methodology, while Section~\ref{sec:methods} introduces the proposed approach. In Section~\ref{sec:results}, we present the experimental evaluation. We conclude the work and discuss future opportunities in Section~\ref{sec:conclusion}.

\section{Related work}
\label{sec:relatedwork}
Matrix factorization (MF) is one of the most popular methods for data embedding, which enables the discovery of interesting feature patterns by clustering and gaining additional knowledge from the resulting factor matrices.
A well-known matrix two-factorization approach is non-negative matrix factorization (\texttt{NMF}), which imposes non-negativity on both the input and output factor matrices for a more straightforward interpretation of the results.
The tri-factorization based \texttt{NMF} called \texttt{NMTF} is used to extract patterns from relational data~\cite{fu2018matrix}, and is applied in various research areas from modeling topics in text data~\cite{wang2011fast} to discovering disease-disease associations~\cite{vzitnik2013discovering}. \texttt{Fast-NMTF}~\cite{vcopar2019fast} is a version of \texttt{NMTF} that uses faster training algorithms based on projected gradients, coordinate descent, and alternating least squares optimization.
One of the usual applications of \texttt{NMTF} is in data fusion methods. \texttt{DFMF}~\cite{datafusion} is a variant of penalized matrix tri-factorization for data fusion, which simultaneously factorizes data matrices in standard linear algebra to reveal hidden associations.

In the field of tropical matrix factorization, De Schutter \& De Moor in 1997~\cite{tropical} presented a heuristic algorithm \texttt{TMF} to compute factorization of a matrix over the tropical semiring. The 
\texttt{STMF} method~\cite{stmf} is based on \texttt{TMF}, but it can perform matrix completion over the tropical semiring. With \texttt{STMF}, we have shown that tropical operations can discover patterns that cannot be revealed with standard linear algebra. \texttt{FastSTMF}~\cite{faststmf} is an efficient version of \texttt{STMF}, where we introduce a faster way of updating factor matrices. The main advantage of \texttt{FastSTMF} over \texttt{STMF} is better computational performance since it achieves better results with less computation. Both \texttt{STMF} and \texttt{FastSTMF} showed the ability to outperform \texttt{NMF} in achieving higher distance correlation and smaller prediction error. However, \texttt{NMF} still achieves better results in terms of approximation error on the train set. 

We can also use matrix factorization to solve different network optimization problems. The Floyd–Warshall algorithm~\cite{hougardy2010floyd} for shortest paths can be formulated as a computation over a $(\min, +)$ semiring. Hook~\cite{linear_regression}, in his work of linear regression over the tropical semiring, showed how a $(\min, +)$ semiring can be used for the low-rank matrix approximation to analyze the structure of a network. The basis of this approach is a two-factorization algorithm that can recover the edge lengths of the shortest path distances for tripartite and bipartite networks. Network partitioning can be done using the algorithm for community detection called the Louvain method~\cite{blondel2008fast}. Another interesting application of semirings is the fact that we can write the Viterbi algorithm~\cite{forney1973viterbi} compactly in a $(\min, +)$ semiring over probabilities~\cite{theodosis2018analysis}.

Currently, no method returns three factorized matrices computed over the tropical semiring. In our work, we propose a first tri-factorization algorithm over the tropical semiring called \texttt{triFastSTMF}, which is based on \texttt{FastSTMF}. To evaluate it empirically, we apply our \texttt{triFastSTMF} to approximate and predict the edge lengths of a four-partition network.

\section{Methods}
\label{sec:methods}
\subsection{Our contribution}

\subsubsection{Semirings $(\max,+)$ and $(\min,+)$}
In a matrix semiring, the operations on the matrices are based on the main operations in the underlying semiring. We denote by $\mathbb{R}_{\max}^{t \times s}$ the set of all matrices with $t$ rows and $s$ columns over $\mathbb{R}_{\max}$ and for a matrix $X \in \mathbb{R}_{\max}^{t \times s}$ we denote its element in the $i$th row and the $j$th column by $X_{ij}$. Moreover, $ \mathbb{R}_{\max}^{t}= \mathbb{R}_{\max}^{t \times 1}$ is the set of all vectors with $t$ components over $\mathbb{R}_{\max}$.
We define the matrix addition over $\mathbb{R}_{\max}$ as
\begin{equation*}
    (A \oplus B)_{ij} = A_{ij} \oplus B_{ij} = \max\{A_{ij}, B_{ij}\},
\end{equation*}
for all $A, B \in \mathbb{R}_{\max}^{m \times n}$, $i=1,...,m$ and $j=1,...,n$, and the matrix multiplication as
\begin{equation*}
    (A \otimes B)_{ij} = \bigoplus\limits_{k=1}^{p} A_{ik} \otimes B_{kj} = \max\limits_{1 \leq k \leq p}\{A_{ik} + B_{kj}\},
\end{equation*}
 for $A \in \mathbb{R}_{\max}^{m \times p}$ and $B \in \mathbb{R}_{\max}^{p \times n}$. Similarly, in the $(\min,+)$ semiring, the matrix addition is defined as
\begin{equation*}
    (A \oplus^* B)_{ij} = A_{ij} \oplus^* B_{ij} = \min\{A_{ij}, B_{ij}\}
\end{equation*}
 for all $A, B \in \mathbb{R}_{\min}^{m \times n}$, $i=1,...,m$ and $j=1,...,n$, and the matrix multiplication is defined as
\begin{equation*}
    (A \otimes^* B)_{ij} = \bigoplus\limits_{k=1}^{p} A_{ik} \otimes^* B_{kj} = \min\limits_{1 \leq k \leq p}\{A_{ik} + B_{kj}\},
\end{equation*}
 for $A \in \mathbb{R}_{\min}^{m \times p}$ and $B \in \mathbb{R}_{\min}^{p \times n}$ for $i=1,...,m$ and $j=1,...,n$. 

We say that matrix $A$ is less than or equal to matrix $B$, denoted as $A \preceq B$, if every element in $A$ is less than or equal to its corresponding element in $B$.
For given matrices $A \in {\mathbb R}_{\max}^{m \times n}$ and \mbox{$B \in {\mathbb{R}}_{\max}^{m \times p}$}, the solutions of matrix equation  
\begin{equation}\label{eq:lineq}
    A \otimes X=B
\end{equation} do not need to exist. However, there might exist some matrices $X'\in {\mathbb{R}}_{\max}^{n \times p}$, such that $A \otimes X'\preceq B$. Such $X'$ is called a \emph{subsolution} of the equation \eqref{eq:lineq}. The \emph{greatest subsolution} of \eqref{eq:lineq} is  a matrix $X_0\in {\mathbb{R}}_{\max}^{n \times p}$, such that $A\otimes X_0 \preceq B$ and for any matrix $X'$, satisfying $A\otimes X' \preceq B$ we have $X' \preceq X_0$.

It is well known (see, \textit{e.g.}~\cite{gaubert}) that for $A \in {\mathbb R}_{\max}^{m \times n}$ and $\bfb=[b_1\, b_2\, \ldots b_m]^T\in  {\mathbb{R}}_{\max}^{m}$, the greatest subsolution $\bfx=[x_1\, x_2\, \ldots x_n]^T\in {\mathbb{R}}_{\max}^{n}$ of $$A\otimes \bfx = \bfb$$ exists and is given by \begin{equation*}
    x_k = -\max_{1\leq \ell \leq m}\{-b_\ell + A_{\ell k}\}=\min_{1\leq \ell \leq m}\{ -A^T_{k \ell}+b_\ell\},
\end{equation*}
for $k=1,\ldots,n$, or equivalently 
\begin{equation*}
    \bfx =-A^T \otimes^* \bfb.
\end{equation*}
More generally, for matrix equations the greatest subsolution is given by the following theorem.

\begin{theorem}[Described by Gaubert and Plus~\cite{gaubert}]\label{thm:GSS}
For any $A \in {\mathbb{R}}_{\max}^{m \times n}$ and $B \in {\mathbb{R}}_{\max}^{m \times p}$ the greatest subsolution of the equation \mbox{$A \otimes X=B$} is 
$$X= (-A)^T \otimes^* B.$$
\end{theorem}

\noindent In what follows, we need to include both operations $\otimes$ and $\otimes^*$ in our computations. First, we prove the following technical lemma.

\begin{lemma}\label{lem:max_min_mult}
For any $A \in {\mathbb{R}}_{\max}^{m \times n}$, $B \in {\mathbb{R}}_{\max}^{n \times p}$ and $C \in {\mathbb{R}}_{\max}^{p \times q}$ we have 
$$(A \otimes B) \otimes^* C= A \otimes (B \otimes^*C)$$ and 
$$(A \otimes^* B) \otimes C= A \otimes^* (B \otimes C).$$
\end{lemma}
\begin{proof}
 For any $k \in \{1,2,\ldots,m\}$ and $\ell \in \{1,2,\ldots,q\}$ we have
 \begin{align*}
    ((A \otimes B) \otimes^* C)_{k \ell}&=\min_{1 \leq i \leq p}\{(A\otimes^* B)_{ki}+C_{i \ell}\}=\\
      &=\min_{1 \leq i \leq p}\max_{1 \leq j \leq n}\{A_{kj}+B_{ji} + C_{i \ell}\}=\\
      &=\max_{1 \leq j \leq n}\min_{1 \leq i \leq p}\{A_{kj}+B_{ji} + C_{i \ell}\}=\\
     &=\max_{1 \leq j \leq n}\{A_{kj}+(B\otimes^* C)_{j \ell}\}=\\
     &=(A \otimes (B \otimes^*C))_{k \ell},
 \end{align*}
 which proves the first equality. Similarly, we prove the second one.
\end{proof}

To implement a tropical matrix tri-factorization algorithm, we need to know how to solve tropical linear systems. In particular, we need to find the greatest subsolution of the linear system 
$A \otimes X \otimes B = C$.

\begin{theorem}\label{thm:gss}
For any $A \in {\mathbb{R}}_{\max}^{m \times n}$, $B \in {\mathbb{R}}_{\max}^{p \times q}$ and $C \in {\mathbb{R}}_{\max}^{m \times q}$ the $n \times p$ matrix
\begin{equation*}
X = (-A)^{T} \otimes^{*} C \otimes^{*} (-B)^{T} 
\end{equation*}
is the greatest subsolution of the equation 
\begin{equation}
    A \otimes X \otimes B = C.
    \label{mixed_equation}
\end{equation}
\end{theorem}

\begin{proof}
Observing the equation $A \otimes Y = C$, its greatest subsolution is by Theorem \ref{thm:GSS} equal to $Y' = (-A)^{T} \otimes^{*} C$, implying
\begin{equation}
    A \otimes ((-A)^{T} \otimes^{*} C)=A \otimes Y' \preceq C.
    \label{first_inequality}
\end{equation}
Moreover, if any matrix $Y''$ satisfies the inequality $A \otimes Y'' \preceq C$, this implies that $Y''\preceq (-A)^T \otimes^* C$. Similarly, the greatest subsolution of the equality $Z \otimes B = C $ is by Theorem \ref{thm:GSS} equal to $Z' = C \otimes^{*} (-B)^{T}$, thus
\begin{equation}
    (C \otimes^{*} (-B)^{T}) \otimes B=Z' \otimes B \preceq C,
    \label{second_inequality}
\end{equation}
and if any matrix $Z''$ satisfies the inequality $Z'' \otimes B \preceq C $, this implies that $Z''\preceq C \otimes^{*} (-B)^{T}$. 

Define $X_0=(-A)^{T} \otimes^{*} C \otimes^{*} (-B)^{T}$. Using equations \eqref{first_inequality}, \eqref{second_inequality} and Lemma \ref{lem:max_min_mult} observe that 
\begin{align*}
    A \otimes X_0 \otimes B &= A \otimes ((-A)^{T} \otimes^{*} C \otimes^{*} (-B)^{T}) \otimes B =\\
    &=(A \otimes (-A)^{T} \otimes^{*} C) \otimes^{*} (-B)^{T} \otimes B\preceq \\
    &\preceq C \otimes^{*} (-B)^{T} \otimes B \preceq C,
\end{align*}
which implies that $X_0=(-A)^{T} \otimes^{*} C \otimes^{*} (-B)^{T}$ is the subsolution of equation~\eqref{mixed_equation}. 

Assume now there exists a subsolution  $X'$ of  \eqref{triFastSTMF}, \textit{i.e.}, $$A \otimes X' \otimes B \preceq C.$$ Let us prove that $X' \preceq X_0$, which will imply that $X_0$ is the greatest subsolution of equation \eqref{triFastSTMF}. Since $X' \otimes B$ is the subsolution of the equation $A \otimes Y=C$, it follows that \mbox{$X' \otimes B \preceq (-A)^T\otimes^* C$}. This implies $X'$ is the subsolution of the equation $Z \otimes B = (-A)^T\otimes C$, which assures that 
\begin{equation*}
X' \preceq (-A)^T\otimes^* C \otimes^* (-B)^T=X_0. \qedhere
\end{equation*}
\end{proof}

\subsubsection{Tri-factorization over the tropical semiring}

We propose a tri-factorization algorithm \texttt{triFastSTMF} over the tropical semiring, which returns three factorized matrices that we later use for the analysis of the structure of four-partition networks.

\emph{Matrix tri-factorization over a tropical semiring} is a decomposition of a form
\mbox{$R= G_1 \otimes S \otimes G_2$}, where
\mbox{$R \in \mathbb{R}_{\max}^{m \times n}$}, $G_1 \in \mathbb{R}_{\max}^{m \times r_1}$,  $S \in \mathbb{R}_{\max}^{r_1 \times r_2}$, $G_2 \in \mathbb{R}_{\max}^{r_2 \times n}$, $r_1 \in \mathbb{N}_0$ and $r_2 \in \mathbb{N}_0$. 
Since for small values of $r_1$ and $r_2$ such decomposition may not exist, we define the tropical matrix tri-factorization problem as:
Given a matrix $R$ and factorization ranks $r_1$ and $r_2$, find matrices $G_1, S$ and $G_2$ such that 
\begin{equation}
    \label{eq:3factorization}
    R\cong G_1 \otimes S \otimes G_2.
\end{equation}

Because the solution of equation \eqref{eq:3factorization} does not exist in general, we will evaluate the computed tri-factorization by \emph{$b$-norm}, defined as $\|W\|_b=\sum_{i,j}\vert W_{ij} \vert$. In particular, we want to minimize the cost function
\begin{equation*}
        J(G; S) = \norm{R - G_1 \otimes S \otimes G_2}_b. 
\end{equation*}{}

In Algorithm~\ref{triFastSTMF}, we present the pseudocode of the algorithm \texttt{triFastSTMF} illustrated in Figure~\ref{schema}. 
The convergence of the proposed algorithm \texttt{triFastSTMF}, defined in Algorithm~\ref{triFastSTMF}, is checked similarly to that of \texttt{STMF}~\cite{stmf} and \texttt{FastSTMF}~\cite{faststmf}.
The factor matrices are updated only if the $b$-norm decreases, ensuring that the approximation error is monotonically reduced.

The \texttt{triFastSTMF} method consists of the following steps:

\begin{figure*}[thp]
\centering
\includegraphics[width=0.55\linewidth]{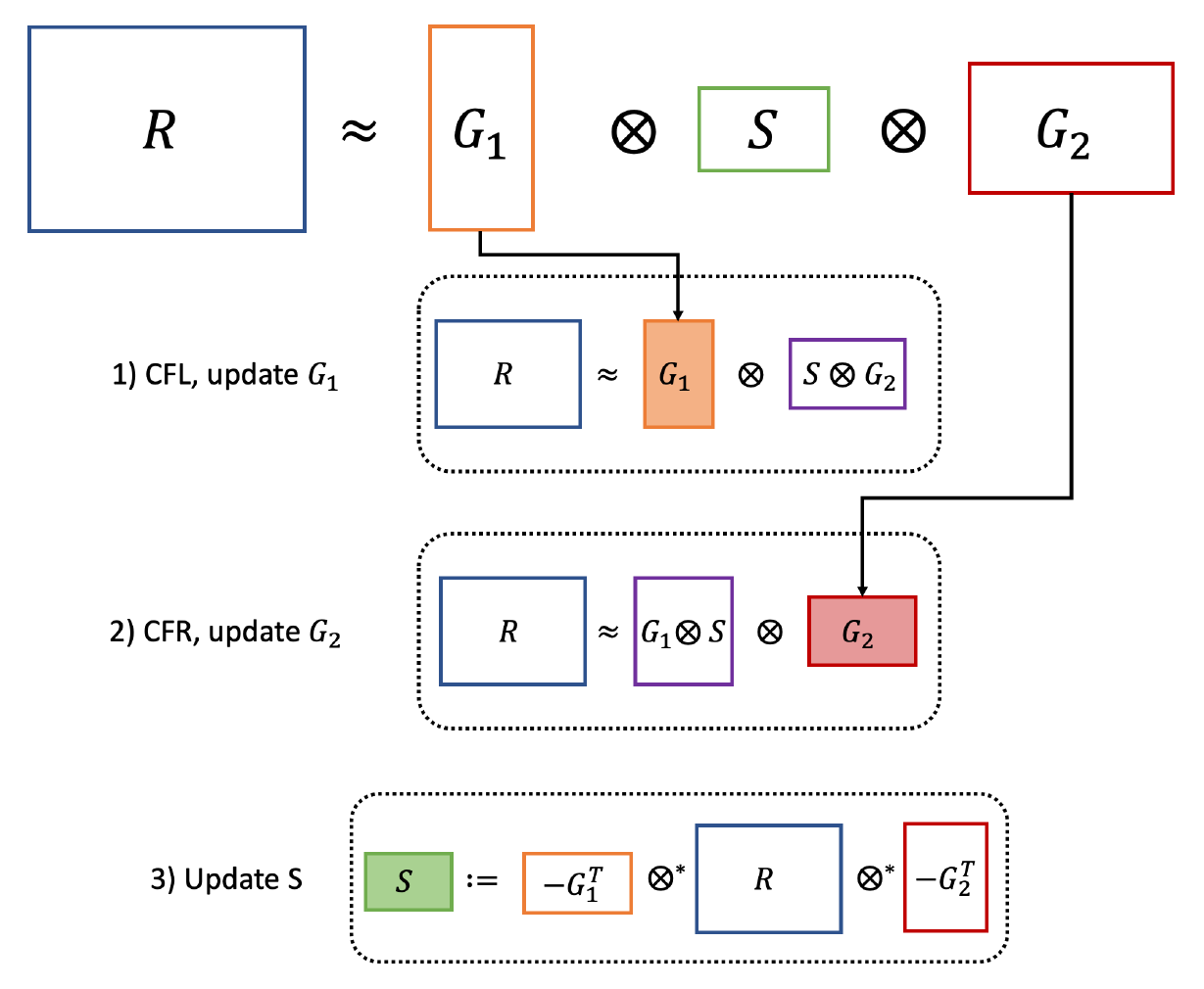}
\caption{Schematic diagram of one iteration of the proposed \texttt{triFastSTMF} method for updating factor matrices $G_1, S$ and $G_2$ of the data matrix $R\approx G_1\otimes S\otimes G_2$.
Step 1) updates the factor matrix $G_1$ through \texttt{CFL}, while step 2) uses the new $G_1$ to update $G_2$ through \texttt{CFR}. The last step, 3) updates $S$ using Theorem~\ref{thm:gss} and newly-computed factor matrices $G_1$ and $G_2$. The procedure repeats until convergence.}
\label{schema}
\end{figure*}

\begin{enumerate}
    \item We follow the results obtained in~\cite{faststmf} to preprocess a data matrix into a suitable shape using transformations, like matrix transposition and random permutation of rows.
    Wide matrices are shown to achieve smaller errors compared to tall matrices~\cite{faststmf}.
    \item The default initialization of factor matrices $G_1$, $S$ and $G_2$ uses the Random Acol strategy~\cite{stmf}, which computes the element-wise average of randomly selected columns from matrix $R$. Fixed initialization for matrices  $G_1$, $S$, and $G_2$ can be used straight from the data, see Section~\ref{subsection:real_data}. \item Until converged, each iteration of the algorithm first updates $G_1$ and $G_2$ using \texttt{CFL} and \texttt{CFR}, presented in Algorithms~\ref{CFL} and~\ref{CFR}, respectively, and described below.
Then we compute the middle factor $S$ as the greatest subsolution of equation $G_1 \otimes S \otimes G_2 = R$ by Theorem~\ref{thm:gss} as
\begin{equation*}
    S = (-G_1)^{T} \otimes^{*} R \otimes^{*} (-G_2)^{T}.
    \end{equation*}
    \item As the last step of \texttt{triFastSTMF}, we reshape the factor matrices $G_1$, $S$ and $G_2$ into appropriate forms depending on the initial transformation of the data matrix $R$. 
    If some of the elements of the data matrix $R$ are not given, we apply the operations proposed in~\cite{stmf} to skip all the missing values in the calculation.
\end{enumerate}

\begin{algorithm}[t]
\caption{Tri-factorization over the tropical semiring (\texttt{triFastSTMF})}
\begin{algorithmic}
\REQUIRE data matrix $R$ $\in \mathbb{R}_{\max}^{m \times n}$, approximation ranks $r_1$, $r_2$ 
\ENSURE factorization $G_1\in \mathbb{R}_{\max}^{m \times r_1}$, $S \in \mathbb{R}_{\max}^{r_1 \times r_2}$, $G_2\in \mathbb{R}_{\max}^{r_2 \times n}$
\STATE \textbf{if} $R$ not wide \textbf{then} transpose $R$
\STATE $perm \gets$ random permutation of indices $1 \dots m$
\STATE $R \gets R[perm, :]$
\STATE initialize $G_1, G_2$  \STATE $S \gets (-G_1)^{T} \otimes^{*} R \otimes^{*} (-G_2)^{T}$ 

\WHILE {not converged}
\STATE $G_1 \gets CFL(R, G_1, S, G_2)$

\STATE $G_2 \gets CFR(R, G_1, S, G_2)$

\STATE $S \gets (-G_1)^{T} \otimes^{*} R \otimes^{*} (-G_2)^{T}$

\ENDWHILE
\STATE \textbf{if} $R$ transposed \textbf{then} 
\STATE $(G_1, S, G_2) \gets (G_2^{T}, S^{T}, G_1[perm^{-1}, :]^{T})$
\STATE \textbf{else} $(G_1, S, G_2) \gets (G_1[perm^{-1}, :], S, G_2)$
\RETURN $G_1, S, G_2$
\end{algorithmic}
\label{triFastSTMF}
\end{algorithm}

\begin{algorithm}[t]
\caption{Compute Factorization to update the Left factor matrix $G_1$ (\texttt{CFL})}
\begin{algorithmic}
\REQUIRE data matrix $R$ $\in \mathbb{R}_{\max}^{m \times n}$, factor matrices: left $G_1 \in \mathbb{R}_{\max}^{m \times r_1}$, middle $S \in \mathbb{R}_{\max}^{r_1 \times r_2}$, and right $G_2 \in \mathbb{R}_{\max}^{r_2 \times n}$ 
\ENSURE left factor matrix $G_1 \in \mathbb{R}_{\max}^{m \times r_1}$
\STATE $Q = S \otimes G_2$
\WHILE {not converged}
\STATE \textbf{for} each \textit{row} $i$ of $R$
\begin{ALC@g}
\STATE $err$, $row\_inds$, $col\_inds \gets \text{TD\_A}(R, G_1, Q, i)$
\STATE \textbf{for} each $j$ \textbf{in} argsort(\textit{err}) in decreasing order
\begin{ALC@g}
\STATE $k\gets \argmax_\ell\left({\rm count_\ell}\left(row\_inds \text{ }\cup \break \text{ } col\_inds[j]\right)\right)$
\STATE $(G_1, Q, G_{1(\cdot k)}',  Q_{k \cdot}') \gets \text{F-ULF}(R, G_1, Q, i, j, k)$
\STATE \textbf{if} $\left\lVert R - G_1 \otimes S \otimes G_2 \right\rVert_b$ decreases \textbf{then} \textbf{break}
\STATE \textbf{else} $(G_{1(\cdot k)}, Q_{k \cdot}) \gets (G_{1(\cdot k)}', Q_{k \cdot}')$
\end{ALC@g}
\begin{ALC@g}
\STATE $(G_1, Q, G_{1(\cdot k)}',  Q_{k \cdot}') \gets \text{F-URF}(R, G_1, Q, i, j, k)$
\end{ALC@g}
\begin{ALC@g}
\STATE \textbf{if} $\left\lVert R - G_1 \otimes S \otimes G_2 \right\rVert_b$ decreases \textbf{then} \textbf{break}
\STATE \textbf{else} $(G_{1(\cdot k)}, Q_{k \cdot}) \gets (G_{1(\cdot k)}', Q_{k \cdot}')$
\end{ALC@g}
\end{ALC@g}
\ENDWHILE
\RETURN $G_1$
\end{algorithmic}
\label{CFL}
\end{algorithm}

\begin{algorithm}[t]
\caption{Compute Factorization to update the Right factor matrix $G_2$ (\texttt{CFR})}
\begin{algorithmic}
\REQUIRE data matrix $R$ $\in \mathbb{R}_{\max}^{m \times n}$, factor matrices: left $G_1 \in \mathbb{R}_{\max}^{m \times r_1}$, middle $S \in \mathbb{R}_{\max}^{r_1 \times r_2}$, and right $G_2 \in \mathbb{R}_{\max}^{r_2 \times n}$ 
\ENSURE right factor matrix $G_2 \in \mathbb{R}_{\max}^{r_2 \times n}$
\STATE $Q = G_1 \otimes S$
\WHILE {not converged}
\STATE \textbf{for} each \textit{row} $i$ of $R$
\begin{ALC@g}
\STATE $err$, $row\_inds$, $col\_inds \gets \text{TD\_A}(R, Q, G_2, i)$
\STATE \textbf{for} each $j$ \textbf{in} argsort(\textit{err}) in decreasing order
\begin{ALC@g}
\STATE $k\gets \argmax_\ell\left({\rm count_\ell}\left(row\_inds \text{ }\cup \break \text{ } col\_inds[j]\right)\right)$
\STATE $(Q, G_2, Q_{\cdot k}',  G_{2(k \cdot)}') \gets \text{F-ULF}(R, Q, G_2, i, j, k)$
\STATE \textbf{if} $\left\lVert R - G_1 \otimes S \otimes G_2 \right\rVert_b$ decreases \textbf{then} \textbf{break}
\STATE \textbf{else} $(Q_{\cdot k}, G_{2(k \cdot)}) \gets (Q_{\cdot k}', G_{2(k \cdot)}')$
\end{ALC@g}
\begin{ALC@g}
\STATE $(Q, G_2, Q_{\cdot k}',  G_{2(k \cdot)}') \gets \text{F-URF}(R, Q, G_2, i, j, k)$
\end{ALC@g}
\begin{ALC@g}
\STATE \textbf{if} $\left\lVert R - G_1 \otimes S \otimes G_2 \right\rVert_b$ decreases \textbf{then} \textbf{break}
\STATE \textbf{else} $(Q_{\cdot k}, G_{2(k \cdot)}) \gets (Q_{\cdot k}', G_{2(k \cdot)}')$
\end{ALC@g}
\end{ALC@g}
\ENDWHILE
\RETURN $G_2$
\end{algorithmic}
\label{CFR}
\end{algorithm}

Note that \texttt{triFastSTMF} updates one factor matrix at a time using \texttt{CFL} and \texttt{CFR}, presented in Algorithms~\ref{CFL} and~\ref{CFR}, respectively. 
They are both based on \texttt{FastSTMF} and represent the two-factorization with \texttt{FastSTMF} core~\cite{faststmf} that contains minor changes:
\begin{itemize}
    \item In \texttt{CFL}/\texttt{CFR}, we remove the initialization of the factor matrices, as they are already initialized at the beginning of \texttt{triFastSTMF}. In \texttt{CFL}, we update only the left factor matrix $G_1$, and declare $Q=S \otimes G_2$ to be the second factor matrix. Similarly, in \texttt{CFR}, we update only the right factor matrix $G_2$ and $Q=G_1 \otimes S$ is the first factor matrix. This approach prevents overfitting factor matrices since the optimization iterates over the left and right factorization. Such a process gives equal importance to both factor matrices, allowing patterns to spread in multiple factor matrices instead of being consolidated in one of them. 
    \item We change the computation of the approximation error. \texttt{FastSTMF} computes the error of two-factorization, while \texttt{CFL}/\texttt{CFR} computes the tri-factorization error using the current factor matrices $G_1$, $S$, and $G_2$.
    \item We do not transpose the matrices nor permute the rows of matrices in \texttt{CFL}/\texttt{CFR} since this is performed as part of \texttt{triFastSTMF}.
\end{itemize}

The functions \texttt{F-ULF}, \texttt{F-URF} and \texttt{TD-A} used in \texttt{CFL} and \texttt{CFR} are the same as in the \texttt{FastSTMF} algorithm~\cite{faststmf}. We present the pseudocode of \texttt{TD-A} in Algorithm~\ref{TD_A}, where the notation of functions used is given in~\cite{faststmf}.

\begin{algorithm}[H]
\caption{\texttt{TD\_A}}
\begin{algorithmic}
\REQUIRE data matrix $R$  $\in \mathbb{R}_{\max}^{m \times n}$, left factor matrix $U$, right factor matrix $V$, row $i$ of $R$
\ENSURE $errors, row\_indices, column\_indices$
\STATE $row\_indices \gets \{\!\!\{ f(i,t)\colon t=1,\ldots,n \}\!\!\}$
\STATE \textit{errors}, $columns\_indices \gets [\,], [\, ]$
\STATE \textbf{for} each \textit{column} $j$ of $R$
\begin{ALC@g}
\STATE $e \gets \td_{col}(R, U, V, j)$
\STATE \textbf{append} $e$ \textbf{to} \textit{errors}
\STATE $col\_indices \gets \{\!\!\{ f(t,j)\colon t=1,\ldots,m \}\!\!\}$
\STATE \textbf{append} $col\_indices$ \textbf{to} $columns\_indices$
\end{ALC@g}
\RETURN $errors, row\_indices, column\_indices$
\end{algorithmic}
\label{TD_A}
\end{algorithm}

\subsubsection{Different aspects of the tri-factorization on networks}
\label{diff_asp_networks}
The four-partition network shown in Figure~\ref{figure_1} is an illustrative example of where we can apply tri-factorization for network structure analysis. We represent the four-partition network with three factor matrices which is the basis of tri-factorization methods. Further, different approaches to four-partition networks can be used depending on the nature of the data and the task that needs to be solved.

For a network $\Gamma$ with a vertex set 
\begin{align*}
V(\Gamma)=&\{x(i)\colon i=1,\ldots,m\}\cup
        \{y(j)\colon j=1,\ldots,r_1\}\cup\\
        & \{w(k)\colon k=1,\ldots,r_2\}\cup
        \{z(\ell)\colon \ell=1,\ldots,n\}
\end{align*}
and an edge set $E(\Gamma)$, we define a matrix $G_1 \in \mathbb{R}_{\max}^{m \times r_1}$ such that $G_{1(ij)}$ represents the weight on the edge from $x(i)$ to $y(j)$, a matrix $S \in \mathbb{R}_{\max}^{r_1 \times r_2}$ where $S_{jk}$ represents the weight on the edge from $y(j)$ to $w(k)$ and a matrix $G_2 \in \mathbb{R}_{\max}^{r_2 \times n}$ where $G_{2(k\ell)}$ represents the weights of the edges from $w(k)$ to $z(\ell)$. Then $R = G_1 \otimes S \otimes G_2$ is the $m \times n$ matrix such that 
\begin{equation*}
    R_{i\ell} = \max_{1\leq j \leq r_1, 1\leq k \leq r_2} ((G_1)_{ij} + S_{jk} + (G_2)_{k\ell}) \end{equation*}is the length of the longest path from $x(i)$ to $z(\ell)$, see Figure~\ref{figure_1}. If a matrix $R$ is given, we can estimate $G_1, S$ and $G_2$ with \texttt{triFastSTMF}. 

\begin{figure}[htp]
\centering
\includegraphics[width=2.9in]{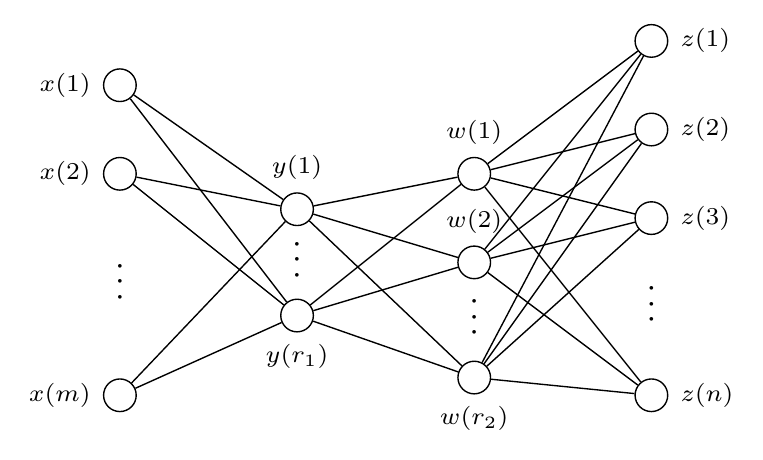}
\caption{Example of a four-partition network.}
\label{figure_1}
\end{figure}

The main question is how to present an arbitrary network as a four-partition network. The two main approaches are:
\begin{itemize}
    \item \textbf{All nodes in the four-partition network are real nodes}. The matrices $G_1, S$, and $G_2$ represent weights of the real edges from the original network, which preserves the interpretability of the network since the relations are only between real nodes. Moreover, the size of the four-partition network remains the same size as the original network. This approach is suitable when the original network's structure already has four partitions. 
    \item \textbf{Some nodes in the four-partition network are latent nodes}. The real nodes are only outer nodes ($x, z$), while latent nodes are inner nodes ($y, w$). In this case, the matrices $G_1, S$ and $G_2$ represent latent features of the outer nodes and not real weights from the original network, leading to a more difficult interpretability of the network since now the relations are also between real and latent nodes. The size of the four-partition network is larger than the size of the original network, which means increases the complexity of the task using this approach.
\end{itemize}
We focus on the first approach, where all nodes in the network are real nodes since we want to use the patterns from the data to initialize the factor matrices, maintain network interpretability, demonstrate how to work with real four-partition networks, and consequently obtain a better approximation of matrices $R, G_1, S, G_2$. In this way, we fully present the power of tri-factorization over two-factorization and its primary purpose.

\subsubsection{Comparison with other strategies}
In our work, we developed different tropical tri-factorization strategies, \texttt{triSTMF} and \texttt{Consecutive}, that are based on two-factorizations~\cite{stmf, faststmf}. We compare their effectiveness with proposed \texttt{triFastSTMF} in Section~\ref{subsubsection_tropmethods}.

The \textbf{triSTMF strategy} is based on the \texttt{TD\_A} method from \texttt{FastSTMF}, and we implement \texttt{triSTMF} tri-factorization as two different two-factorizations: \begin{enumerate}
    \item [\textit{i)}] Left factor matrix is $G_1 \otimes S$, right factor matrix is $G_2$.
    \item [\textit{ii)}] Left factor matrix is $G_1$, right factor matrix is $S \otimes G_2$. 
\end{enumerate} 
We denote errors obtained from \texttt{TD\_A} in the \textit{i)} case as $\varepsilon_L$ and errors in the \textit{ii)} case as $\varepsilon_R$.
We developed two versions called \texttt{triSTMF-BothTD} and \texttt{triSTMF-RandomTD}, which differ in the order of how the error is computed. In \texttt{triSTMF-BothTD}, the computation is performed using both $\varepsilon_L$ and $\varepsilon_R$. The smaller error between $\varepsilon_L$ and $\varepsilon_R$ is selected to perform optimization.
In contrast, \texttt{triSTMF-RandomTD} randomly computes $\varepsilon_L$ or $\varepsilon_R$ and continues with the optimization. Also, \texttt{triSTMF} uses \texttt{ULF} and \texttt{URF} from \texttt{STMF} as the basis for updating factor matrices. Note that we cannot use \texttt{F-ULF} and \texttt{F-URF} directly in the case of tri-factorization since the third factor matrix $S$ introduces additional complexity to \texttt{F-ULF} and \texttt{F-URF}, resulting in incompatible operations. This results in a slow optimization process of both versions of \texttt{triSTMF}.

The \textbf{Consecutive strategy} has two versions: \texttt{lrCon\-se\-cutive} and \texttt{rlConsecutive}.
The goal of this strategy is to achieve tri-factorization by first applying \texttt{FastSTMF} to the data matrix $R$, resulting in factor matrices $U$ and $V$. In the second step, \texttt{lrConsecutive} obtains the third factor matrix by applying \texttt{FastSTMF} to the matrix $V$ to obtain $S$ and $G_2$, while $G_1=U$. In contrast, \texttt{rlConsecutive} applies \texttt{FastSTMF} to the matrix $U$ to obtain $G_1$ and $S$, while $G_2=V$.  The drawback of a consecutive strategy is the consolidation of the patterns in one of the factor matrices during the first step.

\subsection{Synthetic data}
We created a \textit{synthetic data matrix} of size $200 \times 100$ using the $(\max, +)$ multiplication of three random non-negative matrices. Since the purpose of synthetic data is to present the perfect scenario in which the proposed method works the best, we created our synthetic data using three random factor matrices of \textit{sufficiently} large ranks $r_1=25$ and $r_2=20$. We use a synthetic data matrix to compare different tropical matrix factorization methods in Section~\ref{subsubsection_tropmethods}. 
We also created a \textit{synthetic network} with four partitions of sizes $(m, r_1, r_2, n) = (45, 10, 15, 30)$ and use it to analyze four-partition network in Section~\ref{subsubsection:construction}.

\subsection{Real data}
\label{sec:intro_real_data}
We downloaded the real-world interaction dataset of an ant colony~\cite{mersch2013tracking} from the Network Data Repository~\cite{nr-aaai15}. The nodes represent 160 ants, the edges represent physical contact (interaction), and the edge weight is the frequency of interaction during 41 days in total. We preprocessed the network to the appropriate format for evaluation as explained in Section~\ref{subsection:real_data}. In Figure~\ref{figure_real_data}, we show the daily average frequency of interactions between ants. The distance between the nodes indicates the strength of interactions, \textit{i.e.}, nodes are closer when the interaction is stronger; contrary, nodes are farther apart when the interaction is weaker. The outer nodes interact less frequently with the nodes in the center of the network. 
We depict the individual frequency of interactions with the transparency of the edge color in Figure~\ref{figure_real_data}. 
\begin{figure}[!htb]
\centering
\includegraphics[scale=0.6825]{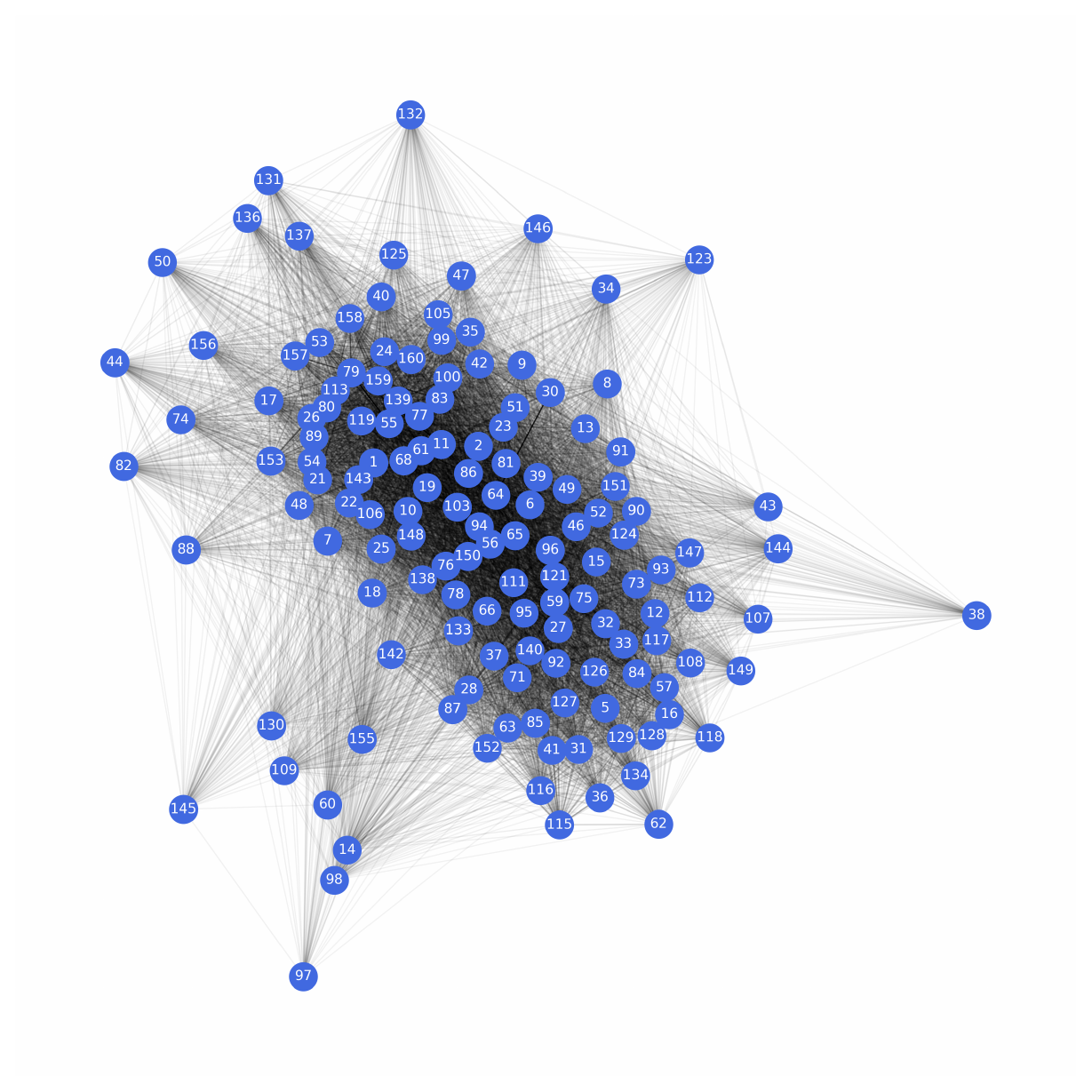}
\caption{A real-world network of the daily average frequency of interactions in an ant colony. The strength of the interaction is visualized with the distance between nodes and edge transparency. 
}
\label{figure_real_data}
\end{figure}

\subsection{Evaluation metrics}
In our work, we use the following metrics:
\begin{itemize}
    \item \textit{Root-mean-square error} or RMSE is a commonly used metric for comparing matrix factorization methods~\cite{faststmf}. We use the  RMSE in our experiments to evaluate the approximation error RMSE-A on the train data, and prediction error RMSE-P on the test data.
    \item \textit{$b$-norm} is defined as $||W||_b=\sum_{i,j}\vert W_{ij} \vert$, and it is used in~\cite{stmf} and~\cite{faststmf} as objective function. We also use the $b$-norm to minimize the approximation error of \texttt{triFastSTMF}.
\item \textit{Rand score} is a similarity measure between two clusterings that considers all pairs of samples and counts pairs assigned in the same or different clusters in the predicted and actual clusterings~\cite{hubert1985comparing}. We use the Rand score to compare different partitioning strategies of the synthetic network.
\end{itemize}

\subsection{Evaluation}
We conducted experiments on synthetic data matrices with true ranks $r_1=25$ and $r_2=20$. The experiments were repeated $25$ times for $300$ seconds using Random Acol initialization.

For the synthetic four-partition network reconstruction, we repeat the experiments $25$ times using fixed initialization with different random and partially-random partitionings. Due to the smaller matrices, these experiments run for $100$ seconds.

For real data, we used the Louvain method~\cite{blondel2008fast} to obtain $r_1$ and $r_2$. Furthermore, we randomly removed at most $20\%$ of the edges. We use fixed initialization and run the experiments for $300$ seconds.

\section{Results}
\label{sec:results}
We perform experiments on synthetic and real data. First, we compare different tropical matrix factorization methods on the synthetic data matrix and show that \texttt{triFastSTMF} achieves the best results of all tropical approaches. Next, we analyze the effect of different partitioning strategies on the performance of \texttt{triFastSTMF}. Finally, we evaluate the proposed \texttt{triFastSTMF} on real data and compare it with \texttt{Fast-NMTF}.

\subsection{Synthetic data}
\subsubsection{Comparison between the tropical matrix factorization methods}
\label{subsubsection_tropmethods}
We experiment with different two-factorization and tri-factorization tropical methods. The set of all tri-factorizations represent a subset of all two-factorizations. Specifically, each tri-factorization is also a two-factorization, meaning that, in general, we cannot obtain better \textit{approximation} results with tri-factorization compared to two-factorization. In Figure~\ref{figure_comp_trop_methods}, we see that the first half of \texttt{lrConsecutive} is better than the second half of \texttt{lrConsecutive}. Namely, in the first half, we perform two-factorization, while in the second half, we factorize one of the factor matrices to obtain three factor matrices as the final result. This second approximation introduces uncertainty and larger errors compared to the first half. We see a similar behavior in \texttt{rlConsecutive}. In this scenario, we show that the two-factorization is better than the tri-factorization. We see that the results of \texttt{triSTMF-BothTD} and \texttt{triSTMF-RandomTD} overlap and do not make any updates during the limited running time since they use slow algorithms to update factor matrices.

Comparing the two-factorization method \texttt{FastSTMF} and the tri-factorization method \texttt{triFastSTMF}, we obtain a similar approximation error in Figure~\ref{figure_comp_trop_methods}. We see that our proposed \texttt{triFastSTMF} achieves the lowest approximation error on the synthetic data matrix of all tested tropical tri-factorization methods. 
\begin{figure}[htp]
\centering
\includegraphics[scale=0.6825]{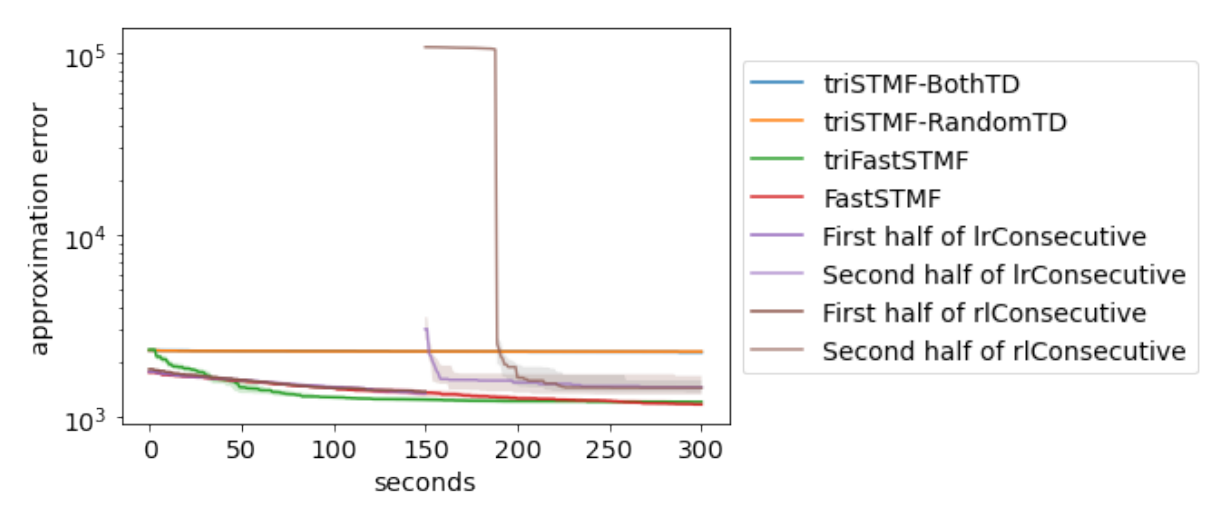}
\caption{Comparison of different tropical tri-factorization methods. The median, first and third quartiles of the approximation error in 25 runs on the synthetic random tropical $200 \times 100$ matrix are shown.}
\label{figure_comp_trop_methods}
\end{figure}
Tri-factorization may outperform two-factorization in a \textit{limited running time} because of the nature of the data and the initialization of factor matrices. Theoretically, we expect that two-factorization and tri-factorization would achieve the same results when evaluated across a large number of datasets.
\noindent Tri-factorization has demonstrated its superiority over two-factorization in many examples. An important application of tri-factorization is the fusion of data from different sources~\cite{datafusion}. In our work, we show that tri-factorization can be applied to approximate and predict weights in four-partition networks.

\subsubsection{Analysis of four-partition network construction}
\label{subsubsection:construction}
We construct a random \textit{tropical} network $K$ of total $100$ nodes with a four-partition  $A\cup B\cup C \cup D$. We denote the sizes of sets $A$, $B$, $C$ and $D$ as $m$, $r_1$, $r_2$ and $n$, respectively, and choose $(m, r_1, r_2, n) = (45, 10, 15, 30)$, see Figure~\ref{figure_synthetic_graph}.
We want to check the robustness of proposed \texttt{triFastSTMF} to the partitioning process and answer the following question: \textit{Is approximation error stable among different choices of partitioning?} 

\noindent Network $K$ contains the following edges:
    \begin{itemize}
        \item edges from $A$ to $B$, denoted as $A-B$, have weights represented by a random matrix $M_1^{m \times r_1}$,
	    \item edges from $B$ to $C$, denoted as $B-C$, have weights represented by a random matrix $T^{r_1 \times r_2}$,
	    \item edges from $C$ to $D$, denoted as $C-D$, have weights represented by a random matrix $M_2^{r_2 \times n}$,	
\item edges from $A$ to $D$, denoted as $A-D$, have weights represented by matrix $E = M_1 \otimes T \otimes M_2$.
    \end{itemize}
We propose the following general algorithm for converting the input network $K$ into a suitable form for tri-factorization. First, partition all network nodes into four sets, $X, Y, W$, and $Z$, with fixed sizes $m, r_1, r_2$ and $n$, respectively, in two ways:
    \begin{itemize}
        \item \textit{random partitioning}: $X\cup Y\cup W\cup Z$ is a random four-partition of the chosen size. Random partitioning 
        is a valid choice when all network nodes represent only one type of object. For example, in a social network, a node represents a person.
\item \textit{partially-random partitioning}: $Y, W$ are random subsets of nodes of $K$ of sizes $r_1$ and $r_2$, while $X=A$ and $Z=D$, where $A, D$ are given. Partially-random partitioning is applicable when there are two types of objects represented in the network. For example, in the movie recommendation system, \textit{users} belong to the set $X$ and \textit{movies} to $Z$. In this case, sets $Y$ and $W$ represent the latent features of $X$ and $Z$.
    \end{itemize}
See examples of random and partially-random partitioning in Figure~\ref{figure_synthetic_graph}, where we show only the edges $X-Y$, $Y-W$ and $W-Z$ to achieve easier readability of the network. Given the (pseudo)random partitioning, construct matrix $R$ as the edges $X-Z$. The matrices $G_1, S$ and $G_2$ are constructed as explained in~\ref{diff_asp_networks} and can be used for the initialization of tri-factorization of $R$ (fixed initialization). For the missing edges, we set the corresponding values in \texttt{triFastSTMF} to be a random number from elements of $G_1, S$ and $G_2$. Tri-factorization on $R$ will return updated $R, G_1, S, G_2$ with approximated/predicted weights on edges.

\begin{figure*}\centering
\includegraphics[width=0.95\linewidth]{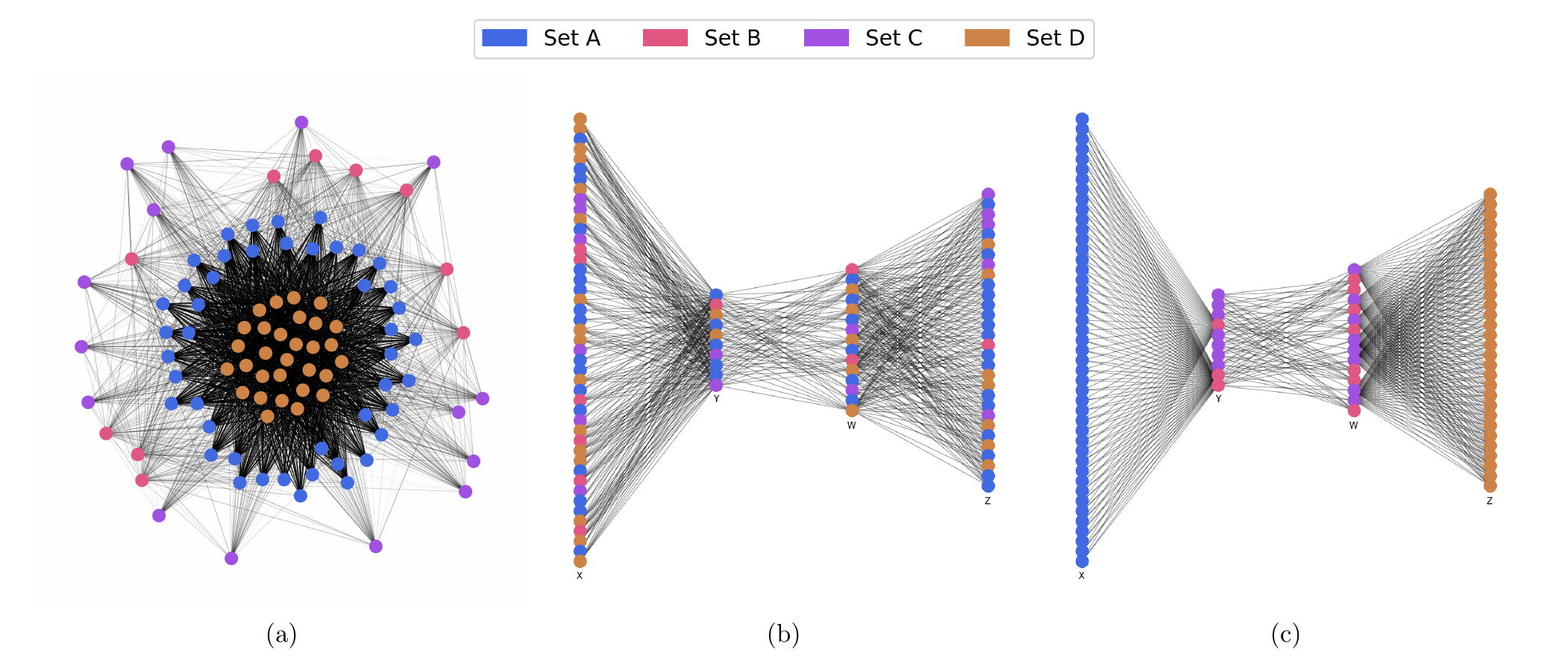}
\caption{(a) A synthetic random \textit{tropical} network $K$ of 100 nodes created by applying the tropical semiring on four sets $A, B, C$ and $D$. The sets $A$ and $D$ are densely connected, following the network construction process. In contrast, sets $B$ and $C$ are less connected. Example of partitioning network $K$, using b) random and c) partially-random partitioning.}
\label{figure_synthetic_graph}
\end{figure*}

We show that partially-random partitioning achieves higher Rand scores, but approximation errors are similar to the ones obtained by random partitioning, see Figure~\ref{rand_score}. We conclude that the partitioning process does not significantly affect the approximation error of \texttt{triFastSTMF}. Still, if there is some additional knowledge about the sets of partition, it is better to use partially-random partitioning. When we do not know the real partition, random partitioning or advanced algorithms, such as the Louvain method, can be used.

\begin{figure}\centering
\includegraphics[scale=0.6825]{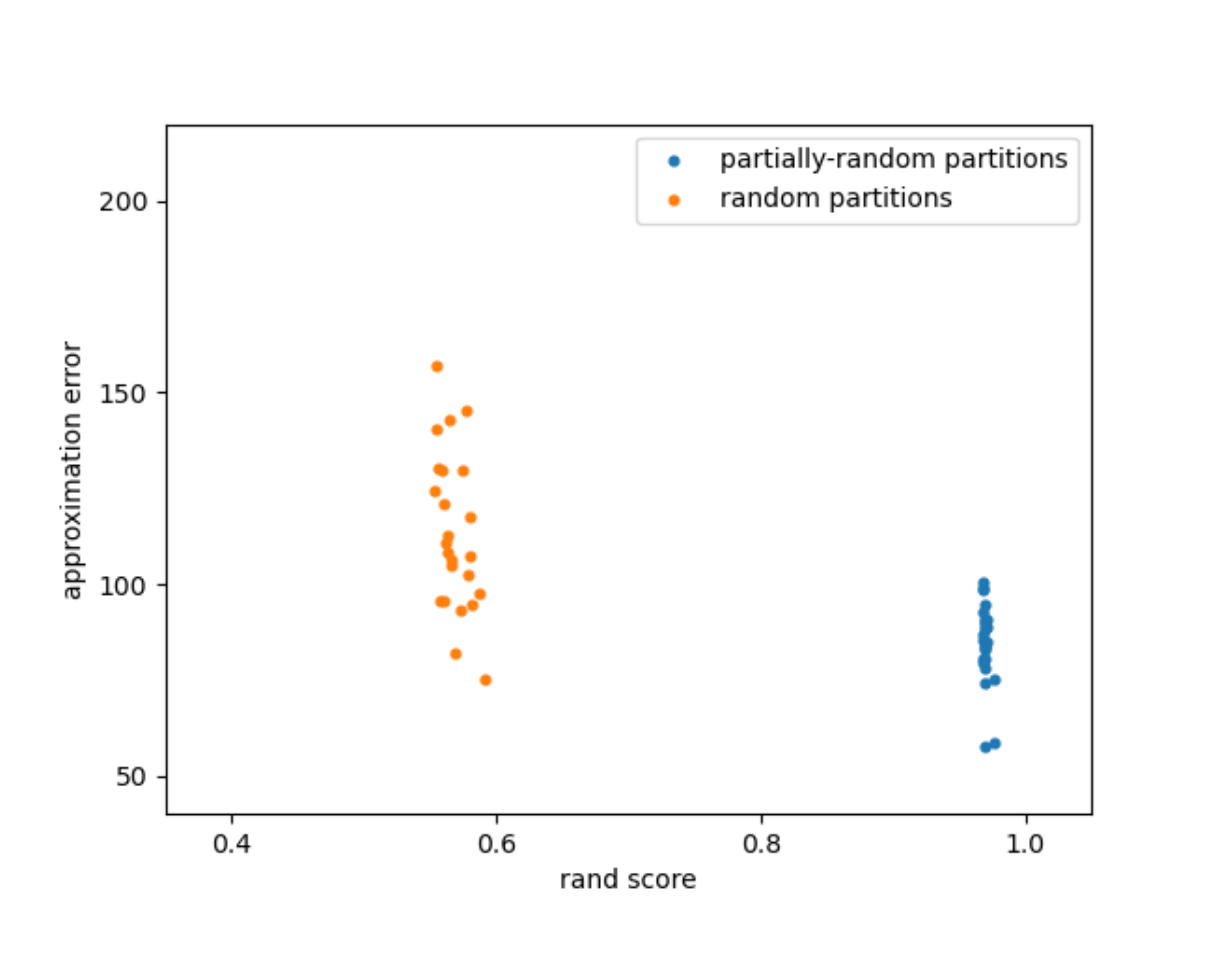}
\caption{Rand score and approximation error of \texttt{triFastSTMF} on 25 random and 25 partially-random partitionings of synthetic data. We performed one run of 100 seconds for each matrix $R$ and used true ranks $r_1$ and $r_2$ as factorization parameters.
}
\label{rand_score}
\end{figure}

\subsection{Real data}
\label{subsection:real_data}

\begin{figure}[!htb]
\centering
\includegraphics[scale=0.6825]{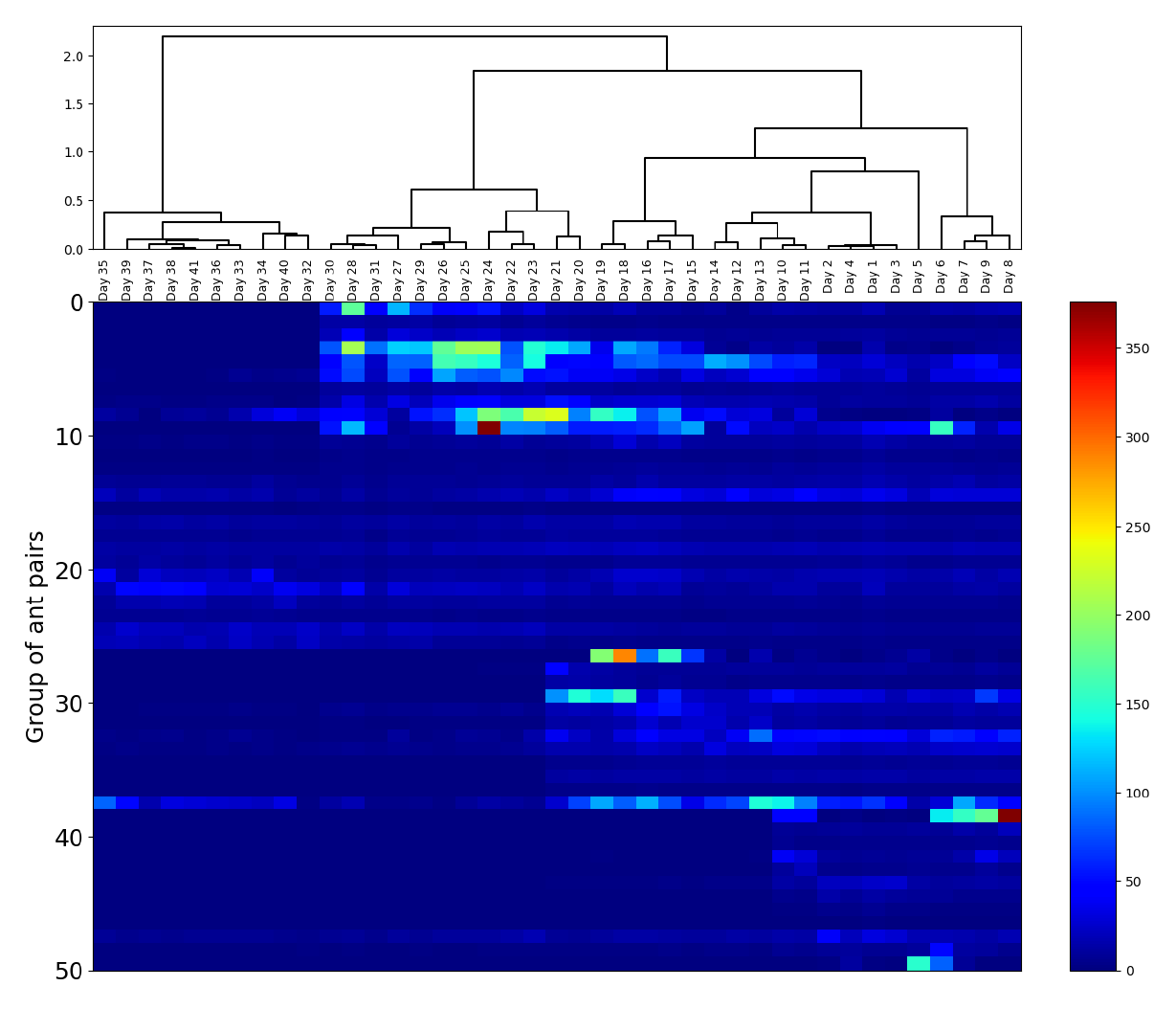}
\caption{Analysis of ants' behavioral patterns over 41 days. The rows represent centroids of clustered ant pairs with k-means using $k=50$, and the columns denote daily interactions. Rows and columns are ordered using Optimal Leaf Ordering for Hierarchical Clustering~\cite{bar2001fast} using cosine distance and Ward linkage.}
\label{figure_heatmap}
\end{figure}

We test our method on a real-world interaction dataset of ant colony introduced in Section~\ref{sec:intro_real_data}.
We describe the data on the interaction between pairs of ants using a weighted adjacency matrix of size $160 \times 160$, where diagonal elements are equal to $0$. The adjacency matrix is symmetric, and we use the data from the upper triangular part to construct the matrix $H$, where each row describes one pair of ants, and columns represent a specific day. Since $H$ is large, we use \textit{k-means} clustering to obtain 50 clusters and analyze the behavioral patterns of the ants on each day, shown in Figure~\ref{figure_heatmap}. 

There are three groups of days with different dynamics of ant interaction: $D_1$ represents days $1-19$, $D_2$ are days $20-31$, and $D_3$ are days $32-41$. We preprocessed the data for each group of days $D_1$, $D_2$ and $D_3$ such that the corresponding weight between two ants represents the daily average of all interactions for the specific days, see Figure~\ref{figure_heatmap_by_days}. 

The group $D_2$, which contains days $20-31$ and $140$ pairs of ants with positive weights, is the most dynamic of the three groups and has local communities. We construct a network $N$ from the group $D_2$, where the nodes represent individual ants, and the weight of the edges represents the strength of interactions between ants. The network density of $N$ is $88\%$. The weighted adjacency matrix of $N$ is denoted as $A$.

\begin{figure*}[!htb]
\centering
\includegraphics[width=0.9\textwidth]{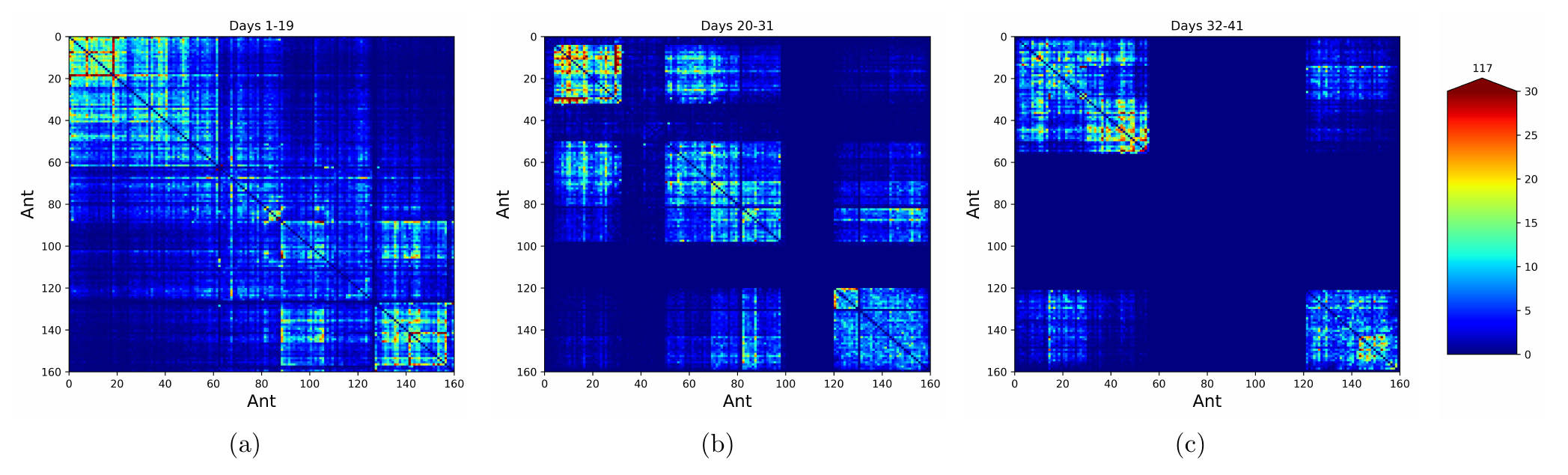}
\caption{Comparison between the daily average of all interactions between ant pairs for different groups of days: (a) days 1-19, (b) days 20-31, and (c) days 32-41. Rows and columns are ordered using Optimal Leaf Ordering for Hierarchical Clustering~\cite{bar2001fast} using cosine distance and Ward linkage.}
\label{figure_heatmap_by_days}
\end{figure*}

Next, we construct ten different networks, $N_1, \dots, N_{10}$ by sampling with replacement the edges from $N$. Each sampled network has at most $20\%$ of missing edges from $N$, which are used for evaluation.
For each network $N_i$, $i \in \{1,\ldots,10\}$, we construct the weighted adjacency matrix $A_i$ with the exact same size and ordering of the nodes in rows and columns as in matrix $A$. Now, to apply tri-factorization on networks, we need to perform Louvain partitioning~\cite{blondel2008fast} for each $N_i$ to obtain a four-partition of its nodes: $X_i\cup Y_i \cup W_i \cup Z_i$.

Louvain method assigns sets of a four-partition and enables favoring larger communities using parameter $\gamma$. Different partitions are obtained for different values of $\gamma$, from which we select a connected four-partition network. We prefer the outer sets $X_i$ and $Z_i$ of corresponding sizes $m$ and $n$, respectively, to have a larger size than the inner sets $Y_i$ and $W_i$ of sizes $r_1$ and $r_2$, respectively. This will ensure that the matrix factorization methods embed data into low-dimensional space using rank values $r_1, r_2 \ll \min\{m, n\}$. Louvain algorithm results in different parameters $m, r_1, r_2$ and $n$ for each $N_i$, $i \in \{1,\ldots,10\}$, shown in Table~\ref{louvain_partitions}. We define $\mu$ to represent a percentage of nodes in outer sets.  Table~\ref{louvain_partitions} shows that $\mu \geq 74\%$ for all $N_i$. We construct $R_i$ matrices of corresponding sizes $m \times n$ using edges from $X_i$ to $Z_i$, and the corresponding matrices $G_1, S$ and $G_2$ of sizes $m \times r_2$, $r_1 \times r_2$ and $r_2 \times n$, respectively, using all four sets. In $R_i$, we mask all values equal to $0$.

\begin{table}[ht]
\centering 
\begin{tabular}{c c c c c c c} \hline
Network & $m$ & $r_1$ & $r_2$ & $n$ & $\mu$ \\
\hline
$N_1$ & $65$ & $21$ & $2$ & $52$  & $84\%$ \\ $N_2$ & $57$ & $22$ & $5$ & $56$  & $81\%$ \\ $N_3$ & $57$ & $24$ & $2$ & $57$ & $81\%$ \\ $N_4$ & $60$ & $21$ & $2$ & $57$ & $84\%$ \\ $N_5$ & $60$ & $20$ & $4$ & $56$ & $83\%$ \\ $N_6$ & $61$ & $19$ & $2$ & $58$ & $85\%$ \\ $N_7$ & $57$ & $18$ & $15$ & $50$ & $76\%$ \\ $N_8$ & $52$ & $23$ & $14$ & $51$ & $74\%$ \\ $N_9$ & $67$ & $4$ & $2$ & $67$ & $96\%$ \\ $N_{10}$ & $65$ & $15$ & $2$ & $58$ & $88\%$ \\ \hline 
\end{tabular}
\caption{Louvain partitioning of $N_i$ where $ i \in [1,10]$, containing $140$ nodes from days $20-31$.} 
\label{louvain_partitions}
\end{table}

We run matrix factorization methods on each $R_i$ matrix using the corresponding factor matrices $G_1, S,$ and $G_2$ for fixed initialization and obtain updated matrices $G_1, S,$ and $G_2$.
Since we use fixed initialization, we evaluate each method only once because there is no presence of randomness. In Table~\ref{table_results_real_data}, we present the comparison between our proposed \texttt{triFastSTMF} and \texttt{Fast-NMTF}. The results show that \texttt{Fast-NMTF} achieves a smaller approximation error RMSE-A, while \texttt{triFastSTMF} outperforms \texttt{Fast-NMTF} in a better prediction error RMSE-P. This result is consistent with previous research in~\cite{stmf} and~\cite{faststmf}, where we have shown that matrix factorization over the tropical semiring is more robust to overfitting compared to methods using standard linear algebra.

\begin{table*}[t]
    \centering
    \scalebox{0.85}{
    \begin{tabular}{c|c|c|c|c|c|c|c|c|c|c|c}
        \textbf{Metric} & \textbf{Method} & $R_1$ & $R_2$ & $R_3$ & $R_4$ & $R_5$ & $R_6$ & $R_7$ & $R_8$ & $R_9$  & $R_{10}$ \\
        \hline
        RMSE-P & \texttt{triFastSTMF} & 1.90 & \textbf{1.07} & 2.32 & \textbf{1.09} & \textbf{1.15} & 1.78 & \textbf{0.88} & \textbf{0.88} & \textbf{1.38} & 2.05 \\
        & \texttt{Fast-NMTF} & \textbf{0.93} & 1.25 & \textbf{1.42} & 1.11 & 1.21 & \textbf{0.82} & 1.17 & 1.10 & 1.48 & \textbf{1.59}\\
        
        \hline
RMSE-A & \texttt{triFastSTMF} & 1.90 & 0.90 & 2.34 & 1.23 & 1.02 & 1.64 & 0.92 & 0.69 & 1.33 & 2.12\\
        & \texttt{Fast-NMTF} & \textbf{0.54} & \textbf{0.37} & \textbf{0.49} & \textbf{0.52} & \textbf{0.43} & \textbf{0.53} & \textbf{0.17} & \textbf{0.17} & \textbf{0.80} & \textbf{0.58}\\
        \hline
    \end{tabular}}
    \caption{RMSE-A and RMSE-P on data matrices $R_i$. The result of the best method in the comparison between \texttt{triFastSTMF} and \texttt{Fast-NMTF} is shown in bold.}
    \label{table_results_real_data}
\end{table*}

\begin{table*}[t]
    \centering
    \scalebox{0.85}{
    \begin{tabular}{c|c|c|c|c|c|c|c|c|c|c|c}
        \textbf{Metric} & \textbf{Method} & $N_1$ & $N_2$ & $N_3$ & $N_4$ & $N_5$ & $N_6$ & $N_7$ & $N_8$ & $N_9$  & $N_{10}$ \\
        \hline
        RMSE-P & \texttt{triFastSTMF} & 9.43 & \textbf{11.40} & 12.76 & 10.35 & \textbf{10.99} & \textbf{8.56} & \textbf{10.63} & \textbf{12.85} & 7.56 & 9.28 \\
        & \texttt{Fast-NMTF} & \textbf{7.24} & 3602725.48 & \textbf{6.21} & \textbf{7.09} & 289129.41 & 116821.55 & 12733965.01 & 763401.17 & \textbf{6.24} & \textbf{6.18}\\
        
        \hline
RMSE-A & \texttt{triFastSTMF} & 8.43 & \textbf{9.95} & 12.37 & 9.63 & \textbf{10.10} & \textbf{8.50} & \textbf{10.20} & \textbf{11.90} & 7.35 & 8.75\\
        & \texttt{Fast-NMTF} & \textbf{6.07} & 1034154.82 & \textbf{6.22} & \textbf{6.02} & 339449.43 & 118555.72 & 13423203.65 & 691714.60 & \textbf{6.02} & \textbf{6.15}\\
        \hline
    \end{tabular}}
    \caption{RMSE-A and RMSE-P on network $N$ using different partitions of $N_i$. The result of the best method in the comparison between \texttt{triFastSTMF} and \texttt{Fast-NMTF} is shown in bold.}
    \label{table_results_real_data_whole_graph}
\end{table*}

The matrix $R_i$ contains only edges $X_i-Z_i$. All other edges $X_i-Y_i$, $Y_i-W_i$ and $W_i-Z_i$ are hidden in the corresponding factor matrices $G_1, S$ and $G_2$. If we want to obtain predictions for all edges of network $N$ using different partitions of $N_i$, we need to also consider factor matrices, not just matrix $R_i$. 
To achieve this, we take into account the corresponding $G_1, S$ and $G_2$ including their products $G_1 \otimes S$, $S \otimes G_2$ and $G_1 \otimes S \otimes G_2$.
The edges that were removed from $N$ during the sampling process to obtain $N_i$ are used to measure the prediction error, while the edges in $N_i$ are used for approximation.

In Table~\ref{table_results_real_data_whole_graph}, we present the comparison between our proposed \texttt{triFastSTMF} and \texttt{Fast-NMTF} on network $N$ using different partitions of $N_i$. The results show that \texttt{triFastSTMF} and \texttt{Fast-NMTF} have the same number of wins regarding the RMSE-A and RMSE-P. However, the main difference between \texttt{triFastSTMF} and \texttt{Fast-NMTF} is in the fact that \texttt{Fast-NMTF} achieves an enormous error compared to \texttt{triFastSTMF} in half of the cases. This is because now we are also predicting edges $X_i-Y_i, Y_i-W_i, W_i-Z_i$ and $X_i-W_i, Y_i-Z_i$, which we obtain by multiplying the corresponding factor matrices $G_1, S$ and $G_2$ properly. There is no guarantee that the factor matrices $G_1, S$, and $G_2$ and their products are on the same scale as the data matrix $R_i$ on which the matrix factorization methods were trained. Since \texttt{Fast-NMTF} uses standard linear algebra, one more matrix multiplication is needed to get to the original data scale. Using standard $+$ and $\times$ operators results in significant error, since the predicted values expand in magnitude quickly. \texttt{triFastSTMF} does not have this problem because it is based on tropical semiring, and the operators $\max$ and $+$ are more averse to predicting large values.

\section{Conclusion}
\label{sec:conclusion}
Matrix factorization is a popular data embedding approach used in various machine learning applications. Most factorization methods use standard linear algebra. Recent research introduced tropical semiring to matrix factorization, which enables the modeling of nonlinear relations. 
Two-factorization approaches are often applied to study bipartite and tripartite networks. However, tri-factorization is suitable for application on four-partition networks, and to the best of our knowledge, our work is the first to explore this option.

In this study, we evaluate different strategies based on two-factorization, called \texttt{triSTMF} and \texttt{Consecutive}. Both strategies have different drawbacks, such as a slow optimization process in \texttt{triSTMF} and the overfitting of one of the factor matrices in \texttt{Consecutive}. These limitations have motivated us to develop a novel tri-factorization approach that addresses the limitations of \texttt{triSTMF} and \texttt{Consecutive}. We propose \texttt{triFastSTMF}, a tri-factorization algorithm over the tropical semiring that can be used for a single data source. Our proposed algorithm is based on \texttt{FastSTMF}, a two-factorization method, with the necessary modifications for tri-factorization. We also provide a detailed theoretical analysis for solving the linear system and computing the third factor matrix. The obtained solution is used for the optimization in the proposed \texttt{triFastSTMF}.

We tested the method on synthetic and real data, applied it to the edge approximation and prediction task in four-partition networks and demonstrated that \texttt{triFastSTMF} achieves close approximation and prediction results as \texttt{Fast-NMTF}. Additionally,  \texttt{triFastSTMF} is more robust than \texttt{Fast-NMTF} in cases when methods are fitted on a part of the network and then used to approximate and predict the entire network.

Although in this study we presented the proposed method on a \textit{single} data source, we established the basis for creating a model capable of combining \textit{multiple} data sources. Our future work involves the application and modification of the proposed \texttt{triFastSTMF} to the data fusion problem, which often employs tri-factorization.

\section*{Supporting information}
The supporting Python notebooks and the data are available on GitHub \url{https://github.com/Ejmric/triFastSTMF}.
We downloaded the real-world interaction dataset of an ant colony named insecta-ant-colony3 from  \textit{Animal Social Networks} data collection on \url{http://networkrepository.com}.

\section*{Author's contributions}
AO, PO and TC designed the study. AO wrote the software application and performed experiments. AO, PO, and TC analyzed and interpreted the results on synthetic and real data. AO wrote the initial draft of the paper. All authors edited and approved the final manuscript.

\section*{Declaration of competing interest}
The authors declare that they have no known competing financial interests or personal relationships that could have appeared to influence the work reported in this paper.

\section*{Funding}
This work is supported by the Slovene Research Agency, Young Researcher grant (52096) awarded to AO, and Research Program funding P1-0222 to PO and P2-0209 to TC.

\bibliographystyle{unsrt}

\bibliography{bibliography}

\end{document}